\documentclass[12pt]{article}

\usepackage{fullpage}
\usepackage{amsmath}
\usepackage{amssymb}
\usepackage{dsfont}
\usepackage{pgfplots}
\usepackage{pgfplotstable}
\usepackage{pifont}
\usepackage{marginnote}

\usepackage{colortbl}
\usepackage[bf]{caption}
\usepackage{wrapfig}
\usepackage{etoolbox}
\usepackage[boxed]{algorithm}
\usepackage{algpseudocode}
\usepackage{times}
\usepackage{wrapfig}
\usepackage{natbib}
\usepackage{hyperref}
\hypersetup{
    bookmarks=true,         
    unicode=false,          
    pdftoolbar=true,        
    pdfmenubar=true,        
    pdffitwindow=false,     
    pdfstartview={FitH},    
    pdftitle={A Scale Free Algorithm for Stochastic Bandits},    
    pdfauthor={Lattimore},     
    pdfsubject={Bandits},   
    pdfcreator={Creator},   
    pdfproducer={Producer}, 
    pdfkeywords={bandits} {statistics} {kurtosis}, 
    pdfnewwindow=true,      
    colorlinks=true,       
    linkcolor=blue,          
    citecolor=blue,        
    filecolor=magenta,      
    urlcolor=cyan           
}

\usepackage{amsthm}
\theoremstyle{plain}
\newtheorem{theorem}{Theorem}

\newtheorem{lemma}[theorem]{Lemma}

\theoremstyle{definition}

\newtheorem{assumption}[theorem]{Assumption}
\newtheorem{remark}[theorem]{Remark}

\theoremstyle{remark}

\usepackage[capitalise]{cleveref}

\newcommand{\eq}[1]{\begin{align*}#1\end{align*}}
\newcommand{\eqn}[1]{\begin{align}#1\end{align}}
\newcommand{\kappaz}{\kappa_{\circ}}
\newcommand{\Prob}[1]{\mathbf{P}\left(#1\right)}
\newcommand{\set}[1]{\left\{ #1\right\}}

\newcommand{\ind}[1]{\mathds{1}\left\{#1\right\}}
\newcommand{\argmax}{\operatornamewithlimits{arg\,max}}

\newcommand{\KL}{\operatorname{KL}}

\renewcommand{\Re}{\mathcal R}

\newcommand{\R}{\mathbf{R}}

\newcommand{\E}{\mathbf{E}}
\newcommand{\Var}{\mathbf{V}}
\newcommand{\Kurt}{\operatorname{Kurt}}

\newcommand{\cH}{\mathcal H}

\newcommand{\cL}{\mathcal L}

\let\epsilon\varepsilon

\title{A Scale Free Algorithm for Stochastic Bandits with Bounded Kurtosis}
\author{Tor Lattimore}

\begin{document}
\maketitle

\begin{abstract}
Existing strategies for finite-armed stochastic bandits mostly depend on a parameter of scale that must be known in advance.
Sometimes this is in the form of a bound on the payoffs, or the knowledge of a variance or subgaussian parameter. The notable exceptions 
are the analysis of Gaussian bandits with unknown mean and variance by \cite{CK15} and of uniform distributions with unknown support \citep{CK15b}. 
The results derived in these specialised cases are generalised here to the non-parametric setup, where the learner knows only a bound on the kurtosis of the noise, which 
is a scale free measure of the extremity of outliers.
\end{abstract}

\section{Introduction}

The purpose of this note is to show that logarithmic regret is possible for finite-armed bandits with no assumptions on the noise of the payoffs except
for a known finite bound on the kurtosis, which corresponds to knowing the likelihood/magnitude of outliers \citep{DeC97}. Importantly, the kurtosis is independent of the location
of the mean and \textit{scale} of the central tendency (the variance). This generalises the ideas of \cite{CK15} beyond the Gaussian case with unknown 
mean and variance to the non-parametric setting.

The setup is as follows. Let $k \geq 2$ be the number of bandits (or arms). In each round $1 \leq t \leq n$ the player 
should choose an action $A_t \in \set{1,\ldots,k}$ and subsequently receives a reward
$X_t \sim \nu_{A_t}$, where $\nu_1,\ldots,\nu_k$ are a set of distributions that are not known in advance.
Let $\mu_i$ be the mean payoff of the $i$th arm and $\mu^* = \max_i \mu_i$
and $\Delta_i = \mu^* - \mu_i$. 
The regret measures the expected deficit of the player relative to the optimal choice of distribution:
\eqn{
\Re_n = \E\left[\sum_{t=1}^n \Delta_{A_t}\right]\,.
\label{eq:regret}
}
The table below summarises many of the known results on the optimal achievable asymptotic regret under different assumptions on $\{\nu_i\}$.
A reference for each of the upper bounds is given in Table~\ref{tab:typical}, while the lower bounds are mostly due to \cite{LR85} and \cite{BK96}.
An omission from the table is when the distributions are known to lie in a single-parameter exponential family (which does not fit well with the columns). 
Details are by \cite{KKM13,Kau16}.
\begin{table}[H]
\renewcommand{\arraystretch}{1.8}
\centering
\small
\begin{tabular}{|cp{4cm}lll|}
\hline
& \textbf{Assumption} & \textbf{Known} & \textbf{Unknown} & $\lim_{n\to\infty} \Re_n / \log(n)$ \\ \hline
1 & Bernoulli \newline \scriptsize \cite{LR85} & $\operatorname{Supp}(\nu_i) \subseteq \set{0,1}$ & $\mu_i \in [0,1]$ & $\displaystyle \sum_{i:\Delta_i > 0} \frac{1}{d(\mu_i, \mu^*)}$ \\
2 & Bounded \newline \scriptsize \cite{HT10} & $\operatorname{Supp}(\nu_i) \subseteq [0,1]$ & distribution & \text{it's complicated}  \\
3 & Semi-bounded \newline \scriptsize \cite{HT15} & $\operatorname{Supp}(\nu_i) \subseteq (-\infty,1]$ & distribution & \text{it's complicated} \\
4 & Gaussian (known var.) \newline \scriptsize \cite{KR95} & $\nu_i = \mathcal N(\mu_i,\sigma_i^2)$ & $\mu_i \in \R$ & $\displaystyle \sum_{i:\Delta_i > 0} \frac{2\sigma^2_i}{\Delta_i}$ \\
5 & Uniform \newline \scriptsize \cite{CK15b} & $\nu_i = \mathcal U(a_i, b_i)$ & $a_i,\,  b_i$ & $\displaystyle \sum_{i:\Delta_i > 0} \frac{\Delta_i}{\log\left(1 + \frac{2\Delta_i}{b_i-a_i}\right)}$ \\
6 & Subgaussian \newline \scriptsize \cite{BC12} & $\log M_{\nu_i}(\lambda) \leq \frac{\lambda^2 \sigma_i^2}{2}\,\, \forall \lambda$ & distribution & \cellcolor{gray!20!white} $\displaystyle \sum_{i:\Delta_i > 0} \frac{2\sigma^2_i}{\Delta_i}$ \\
7 & Known variance \newline \scriptsize \cite{BCL13} & $\Var[\nu_i] \leq \sigma^2_i$ & distribution & \cellcolor{gray!20!white}$\displaystyle O\left(\sum_{i:\Delta_i > 0} \frac{\sigma^2_i}{\Delta_i}\right)$ \\ 
8 & Gaussian \newline \scriptsize \cite{CK15} & $\nu_i = \mathcal N(\mu_i, \sigma^2)$ & $\mu_i \in \R$, $\sigma^2_i > 0$ & $\displaystyle \sum_{i:\Delta_i > 0} \frac{2\Delta_i}{\log\left(1 + \Delta_i^2/\sigma^2_i\right)}$ \\ \hline
\multicolumn{5}{|p{15.3cm}|}{$d(p,q) = p\log(p/q) + (1-p)\log((1-p)/(1-q))$ and 
$M_\nu(\lambda) = \E_{X \sim \nu} \exp((X - \mu) \lambda)$ with $\mu$ the mean of $\nu$ is the centered moment generating function.
All asymptotic results are optimal except for the grey cells.} \\ \hline
\end{tabular}
\caption{Typical distributional assumptions and asymptotic regret}\label{tab:typical}
\end{table}

With the exception of rows 5 and 8 in Table~\ref{tab:typical}, all entries depend on some kind of scale parameter.
Missing is an entry for a non-parametric assumption that is scale free. This paper fills that gap with the following assumption and regret guarantee.

\begin{assumption}\label{assumption1}
There exists a known $\kappa \in \R$ such that for all $1 \leq i \leq k$, the kurtosis of $X \sim \nu_i$ is at most
\eq{
\Kurt[X] = \frac{\E[(X - \E[X])^4]}{\Var[X]^2} \leq \kappa\,.
}
\end{assumption}

\begin{theorem}\label{thm:upper}
If Assumption 1 holds, then the algorithm described in \S\ref{sec:alg} satisfies 
\eq{
\limsup_{n\to\infty} \frac{\Re_n}{\log(n)} \leq C\sum_{i:\Delta_i > 0} \Delta_i \left(\kappa - 1 + \frac{\sigma^2_i}{\Delta_i^2}\right)\,,
}
where $\sigma^2_i$ is the variance of $\nu_i$ and $C > 0$ is a universal constant.
\end{theorem}

What are the implications of this result?
The first point is that the algorithm in \S\ref{sec:alg} is scale and translation invariant in the sense that its behaviour does not change if the
payoffs are multiplied by a positive constant or shifted.
The regret also depends appropriately on the scale so that multiplying the rewards by a positive constant factor also multiplies the regret by this factor.
As far as I know, this is the first scale free bandit algorithm with logarithmic regret on a non-parametric class.
The assumption on the boundedness of the kurtosis is much less restrictive than assuming an exact Gaussian model (which has kurtosis 3) or uniform (kurtosis 9/5). 
See Table \ref{table:kurtosis} for other examples.

\begin{wraptable}{r}{7.8cm}
\renewcommand{\arraystretch}{1.6}
\centering
\begin{tabular}{|lll|}
\hline 
\textbf{Distribution} & \textbf{Parameters} & \textbf{Kurtosis} \\
Gaussian & $\mu \in \R, \sigma^2 > 0$ & 3 \\
Bernoulli & $\mu \in [0,1]$ & $\frac{1 - 3\mu(1-\mu)}{\mu(1-\mu)}$ \\
Exponential & $\lambda > 0$ & $9$ \\ 
Laplace & $\mu \in \R, b > 0$ & $9$ \\
Uniform & $a < b \in \R$ & $9/5$ \\ \hline
\end{tabular}
\caption{Kurtosis}\label{table:kurtosis}
\end{wraptable}
As mentioned, the kurtosis is a measure of the likelihood/existence of outliers of a distribution, and it makes intuitive sense that a bandit strategy might depend
on some kind of assumption on this quantity. How else to know whether or not to cease exploring an unpromising action?
The assumption can also be justified from a mathematical perspective. If the variance of an arm is not assumed known, then calculating confidence intervals
requires an estimate of the variance from the data.
Let $X, X_1,X_2,\ldots,X_n$ be a sequence of i.i.d.\ centered random variable with finite-variance $\sigma^2$. 
A reasonable estimate of $\sigma^2$ is
\eqn{
\hat \sigma^2 = \frac{1}{n} \sum_{t=1}^n X_t^2\,.
\label{eq:simple-est}
}
Clearly this estimator is unbiased and has variance
\eq{
\Var[\hat \sigma^2] = \frac{\E[X^4] - \E[X^2]^2}{n} 
= \frac{\sigma^4 \left(\kappa - 1\right)}{n}\,.
}
Therefore, if we are to expect good estimation of $\sigma^2$, then the kurtosis should be finite.
Note that if $\sigma^2$ is estimated by (\ref{eq:simple-est}), then the central limit theorem combined with 
finite kurtosis is enough for an estimation error of $O(\sigma^2 ((\kappa-1) / n)^{1/2})$ \textit{asymptotically}.
For bandits, however, finite-time bounds are required, which are not available using (\ref{eq:simple-est}) without additional moment assumptions 
(for example, on the moment generating function).
Finite kurtosis alone \textit{is enough} if the classical empirical estimator is replaced by a robust estimator such as the median-of-means estimator \citep{AMS96} or
Catoni's estimator \citep{Cat12}.

\paragraph{Contributions}
The main contribution is the new assumption, algorithm, and the proof of Theorem \ref{thm:upper} (see \S\ref{sec:alg}). The upper bound
is also complemented by a lower bound (\S\ref{sec:lower}).

\paragraph{Additional notation}
Let $T_i(t) = \sum_{t=1}^n \ind{A_t = i}$ be the number of times arm $i$ has been played after round $t$.
For measures $P,Q$ on the same probability space, $\KL(P, Q)$ is the relative entropy between $P$ and $Q$ and $\chi^2(P,Q)$ is the $\chi^2$ distance.
The following lemma is well known.
\begin{lemma}\label{lem:kurtosis}
Let $X_1, X_2$ be independent random variables with $X_i$ having variance $\sigma^2_i$ and kurtosis $\kappa_i < \infty$ and skewness $\gamma_i = \E[(X_i - \E[X_i])^3 / \sigma_i^3]$, then:
\begin{enumerate}
\item[(a)] $\displaystyle \Kurt[X_1 + X_2] = 3 + \frac{\sigma_1^4(\kappa_1 - 3) + \sigma_2^4(\kappa_2 - 3)}{\left(\sigma_1^2 + \sigma_2^2\right)^2}$\,.
\item[(b)] $\displaystyle \gamma_1 \leq \sqrt{\kappa_1 - 1}$\,. 
\end{enumerate}
\end{lemma}

\section{Algorithm and upper bound}\label{sec:alg}

\newcommand{\mm}{\operatorname{\widehat{MM}}}

Like the robust upper confidence bound algorithm by \cite{BCL13}, the new algorithm makes use of the robust median-of-means estimator.

\paragraph{Median-of-means estimator}
Let $Y_1,Y_2,\ldots,Y_n$ be a sequence of independent and identically distributed random variables.
The median-of-means estimator first partitions the data into $m$ blocks of equal size (up to rounding errors). The empirical mean of each block is then computed
and the estimate is the median of the means of each of the blocks. The number of blocks depends on the desired confidence level and should be $O(\log(1/\delta))$.
The median-of-means estimator at confidence level $\delta \in (0,1)$ is denoted by $\mm_\delta(\{Y_t\}_{t=1}^n)$.

\begin{lemma}[\citeauthor{BCL13} \citeyear{BCL13}]\label{lem:mofm}
Let $Y_1,Y_2,\ldots,Y_n$ be a sequence of independent and identically distributed random variables with mean $\mu$ and variance $\sigma^2 < \infty$.
\eq{
\Prob{\left|\mm_\delta\left(\set{Y_t}_{t=1}^n\right) - \mu\right| \geq C_1 \sqrt{\frac{\sigma^2}{n} \log\left(\frac{C_2}{\delta}\right)}} \leq \delta\,,
}
where $C_1 = \sqrt{12\cdot 16}$ and $C_2 = \exp(1/8)$ are universal constants.
\end{lemma}

\paragraph{Upper confidence bounds}
The algorithm is an obvious generalisation of UCB, but with optimistic estimates of the mean and variance.
Let $\delta \in (0,1)$ and $Y_1,Y_2,\ldots,Y_t$ be a sequence of independent and identically distributed random variables with mean $\mu$, variance $\sigma^2$ and kurtosis $\kappa < \infty$.
Furthermore, let
\eq{
\tilde \mu(\{Y_s\}_{s=1}^t, \delta) = \sup\set{\theta \in \R : \theta \leq \mm\left(\set{Y_s}_{s=1}^t\right) + C_1 \sqrt{\frac{\tilde \sigma^2_t(\{Y_s\}_{s=1}^t, \theta, \delta)}{t} \log\left(\frac{C_2}{\delta}\right)}}\,.
}
where 
$\displaystyle \tilde \sigma_t^2(\{Y_s\}_{s=1}^t, \theta, \delta) = \frac{\mm\left(\set{(Y_s - \theta)^2}_{s=1}^t\right)}{\max\set{0,\, 1 - C_1 \sqrt{\frac{\kappa - 1}{t} \log\left(\frac{C_2}{\delta}\right)}}}$\,.

Note that $\tilde \mu(\{Y_s\}_{s=1}^t, \delta)$ may be (positive) infinite if $t$ is insufficiently large.
The following two lemmas show that $\tilde \mu$ is indeed optimistic with high probability, and also that it concentrates with reasonable speed around the true mean.

\begin{lemma}\label{lem:conc1}
$\displaystyle \Prob{\tilde \mu(\{Y_s\}_{s=1}^t, \delta) \leq \mu} \leq 2\delta$\,.
\end{lemma}

\begin{proof}
Apply a union bound and Lemma~\ref{lem:mofm}.
\end{proof}

\begin{lemma}\label{lem:conc2}
Let $\delta_t$ be monotone decreasing and $\tilde \mu_t = \tilde \mu(\{Y_s\}_{s=1}^t, \delta_t)$. Then there exists a universal constant $C_3$ such that for any $\epsilon > 0$, 
\eq{
\sum_{t=1}^n \Prob{\tilde \mu_t \geq \mu + \epsilon} \leq
C_3 \max\set{\kappa - 1,\, \frac{\sigma^2}{\epsilon^2}} \log\left(\frac{C_2}{\delta_n}\right) + 2\sum_{t=1}^n \delta_t\,.
}
\end{lemma}

\begin{proof}
First, by Lemma~\ref{lem:mofm}
\eqn{
\sum_{t=1}^n \Prob{\left|\mm\left(\set{Y_s}_{s=1}^t\right) - \mu\right| \geq C_1 \sqrt{\frac{\sigma^2}{t} \log\left(\frac{C_2}{\delta_t}\right)}} 
&\leq \sum_{t=1}^n \delta_t\,.
\label{eq:conc2-1}
}
Similarly,
\eqn{
\sum_{t=1}^n \Prob{\left|\mm\left(\set{(Y_s-\mu)^2}_{s=1}^t\right) - \sigma^2\right| \geq C_1 \sigma^2 \sqrt{\frac{\kappa - 1}{t} \log\left(\frac{C_2}{\delta}\right)}} 
&\leq \sum_{t=1}^n \delta_t\,.
\label{eq:conc2-2}
}
Suppose that $t$ is a round where all of the following hold:
\begin{enumerate}
\item[(a)] $\displaystyle \left|\mm\left(\set{Y_s}_{s=1}^t\right) - \mu\right| < C_1\sqrt{\frac{\sigma^2}{t} \log\left(\frac{C_2}{\delta_t}\right)}$\,.
\item[(b)] $\displaystyle \left|\mm\left(\set{(Y_s - \mu)^2}_{s=1}^t\right) - \sigma^2\right| < C_1 \sigma^2\sqrt{\frac{\kappa - 1}{t} \log\left(\frac{C_2}{\delta_t}\right)}$\,.
\item[(c)] $\displaystyle t \geq 16C_1^2 (\kappa-1) \log\left(\frac{C_2}{\delta_t}\right)$\,.
\end{enumerate}
Abbreviating $\tilde \sigma^2_t = \tilde \sigma^2(\{Y_s\}_{s=1}^t, \tilde \mu_t, \delta_t)$ and $\hat \mu_t = \mm\left(\set{Y_s}_{s=1}^t\right)$,
\eq{
\tilde \sigma^2_t
&= \frac{\mm\left(\set{(Y_s - \tilde \mu_s)^2}_{s=1}^t\right)}{1 - C_1 \sqrt{\frac{\kappa - 1}{t} \log\left(\frac{C_2}{\delta_t}\right)}} 
\leq 2\mm\left(\set{(Y_s - \tilde \mu_t)^2}_{s=1}^t\right) \\ 
&\leq 4\mm\left(\set{(Y_s - \mu)^2}_{s=1}^t\right) + 4(\tilde \mu_t - \mu)^2 \\
&\leq 4\mm\left(\set{(Y_s - \mu)^2}_{s=1}^t\right) + 8(\tilde \mu_t - \hat \mu_t)^2 + 8(\hat \mu_t - \mu)^2 \\
&< 4\sigma^2 + 4C_1 \sigma^2 \sqrt{\frac{\kappa-1}{t} \log\left(\frac{C_2}{\delta_t}\right)} + \frac{8C_1^2(\sigma^2+\tilde \sigma^2_t)(\kappa - 1)}{t} \log\left(\frac{C_2}{\delta_t}\right) 
\leq \frac{11}{2} \sigma^2 + \frac{\tilde \sigma^2_t}{2}\,,
}
where the first inequality follows from (c), the second since $(x - y)^2 \leq 2x^2 + 2y^2$ and the fact that
\eq{
\mm(\{a Y_s + b\}_{s=1}^t = a\mm(\{Y_s\}_{s=1}^t) + b\,.
}
The third inequality again uses $(x - y)^2 \leq 2x^2 + 2y^2$, while the last uses the definition of $\tilde \mu_t$ and (b).
Therefore $\tilde \sigma^2_t \leq 11\sigma^2$, which means that if (a--c) and additionally
\begin{enumerate}
\item[(d)] $\displaystyle t \geq \frac{19C_1^2 \sigma^2}{\epsilon^2} \log\left(\frac{1}{\delta_n}\right)$\,.
\end{enumerate}
Then
\eq{
|\tilde \mu_t - \mu|
&\leq |\tilde \mu_t - \hat \mu_t| + |\hat \mu_t - \mu| 
< C_1 \sqrt{\frac{\tilde\sigma_t^2}{t} \log\left(\frac{C_2}{\delta_n}\right)} + C_1 \sqrt{\frac{\sigma^2}{t} \log\left(\frac{C_2}{\delta_n}\right)} \\
&\leq C_1 \sqrt{\frac{11\sigma^2}{t} \log\left(\frac{C_2}{\delta_n}\right)} + C_1 \sqrt{\frac{\sigma^2}{t} \log\left(\frac{C_2}{\delta_n}\right)} 
\leq \epsilon\,.
}
Combining this with (\ref{eq:conc2-1}) and (\ref{eq:conc2-2}) and choosing $C_3 = 19 C_1^2$ completes the result.
\end{proof}

\paragraph{Algorithm}
The new algorithm simply uses the upper confidence bound in the last section.
Let $\delta_t = 1/(t^2 \log(1+t))$ and
\eq{
\tilde \mu_i(t) = \tilde \mu_i(\{X_s\}_{s\in[t], A_s = i}, \delta_t) \in (-\infty, \infty]\,.
}
In each round the algorithm chooses $A_t = \argmax_{i \in [k]} \tilde \mu_i(t-1)$,
where ties are broken arbitrarily.

\begin{proof}[Proof of Theorem \ref{thm:upper}]
Assume without loss of generality that $\mu_1 = \mu^*$.
The regret is
\eqn{
\Re_n = \sum_{i=1}^k \Delta_i \E[T_i(n)]\,.
\label{eq:decomp}
}
A bound on $\E[T_i(n)]$ follows immediately from Lemmas~\ref{lem:conc1} and \ref{lem:conc2}.
\eq{
\E[T_i(n)]
&\leq \sum_{t=1}^n \Prob{\tilde \mu_1(t-1) \leq \mu_1} + \sum_{t=1}^n \Prob{\tilde \mu_i(t-1) \geq \mu_1 \text{ and } A_t = i}
}
The first term is bounded using Lemma \ref{lem:conc1}.
\eq{
\sum_{t=1}^n \Prob{\tilde \mu_1(t-1) \leq \mu_1} 
&\leq \sum_{t=1}^n \sum_{u=1}^t \Prob{\tilde \mu_1(t-1) \leq \mu_1 \text{ and } T_1(t-1) = u} \\
&\leq 2\sum_{t=1}^n \sum_{u=1}^t \delta_t 
= 2\sum_{t=1}^n t \delta_t = o(\log(n))\,. 
}
The second term is bound using Lemma \ref{lem:conc2}. 
\eq{
\sum_{t=1}^n \Prob{\tilde \mu_i(t-1) \geq \mu_1 \text{ and } A_t = i}
&\leq \sum_{t=1}^n \Prob{\tilde \mu_i(t-1) - \mu_i \geq \Delta_i} \\
&\leq C_3 \max\set{\kappa - 1, \frac{\sigma^2_i}{\Delta_i^2}} \log\left(\frac{C_2}{\delta_n}\right) + \underbrace{2 \sum_{t=1}^n \delta_t}_{o(\log(n))}\,.
}
Combining the last two displays with (\ref{eq:decomp}) completes the proof.
\end{proof}

\section{Lower bounds}\label{sec:lower}

I briefly present some lower bounds. For the remainder, assume a fixed bandit strategy.
We need two sets of distributions on $\R$. 
\eq{
\cH_{\sigma} &= \set{\nu : \nu \text{ is $\sigma^2$-subgaussian}}\,. \\
\cH_{\kappa} &= \set{\nu : \nu \text{ has kurtosis less than } \kappa}\,.
}
Following the nomenclature of \cite{LR85}, a bandit strategy is called consistent over a set of distributions $\cH$ if $\Re_n = o(n^p)$ for 
all $p \in (0,1)$ and bandits in $\cH^k$. I call a bandit $\{\nu_i\}$ is non-trivial if there exists a suboptimal arm. 
The first theorem shows that if a strategy is consistent over
$\bigcup_{\sigma \geq 0} \cH_{\sigma}^k$, then it does not enjoy logarithmic regret on any non-trivial bandit.
The proof is quite standard and is simply omitted.

\begin{theorem}\label{thm:lower}
Suppose there exists a $\sigma > 0$ and non-trivial bandit $\{\nu_i\} \in \cH_{\sigma}^k$ such that 
\eq{
\limsup_{n\to\infty} \frac{\Re_n}{\log(n)} < \infty\,.
}
Then the strategy is not consistent over $\bigcup_{\sigma \geq 0} \cH_{\sigma}^k$.
\end{theorem}

\begin{remark}
There \textit{are} consistent strategy over $\bigcup_{\sigma \geq 0} \cH_\sigma^k$. For example, let $f(t)$ be a monotone increasing function with 
$f(t) = \omega(\log(t))$ and
$f(t) = o(t^p)$ for all $p \in (0,1)$ and consider the strategy that maximises the following index.
\eq{
\hat \mu_i(t-1) + \sqrt{\frac{f(t)}{T_i(T-1)}}.
}
By following the analysis in Chapter 2 of the book by \cite{BC12} and noting that for $t$ sufficiently large $f(t) \geq 2\sigma^2 \log(t)$, it is easy to show that this strategy satisfies
\eq{
\limsup_{n\to\infty} \frac{\Re_n}{f(t)} = \sum_{i:\Delta_i > 0} \frac{1}{\Delta_i^2}
}
for any bandit $\{\nu_i\} \in \cH_{\sigma}^k$.
It is important to emphasise that the asymptotics here hide large constants that depend on $\tau = \min\{t : f(t) \geq 2\sigma^2 \log(t)\}$.
\end{remark}

The next theorem shows that the upper bound derived in the previous section is nearly tight up to constant factors.
Like most lower bounds, the proof relies on understanding the information geometry of the set of possible distributions.
Let $\cH$ be a family of distributions and let $\{\nu_i\}$ be a non-trivial bandit and $i$ be a suboptimal arm. \cite{BK96} showed that for any consistent strategy
\eqn{
\label{eq:lower}
\liminf_{n\to\infty} \frac{\E[T_i(n)]}{\log(n)} \geq \sup\set {\frac{1}{\KL(\nu_i, \nu'_i)} : \nu_i' \in \cH \text{ and } \E_{X \sim \nu'_i}[X] > \mu^*}\,.
}
In parameterised families of distributions, the optimisation problem can often be evaluated analytically (eg., Bernoulli, Gaussian with known variance, Gaussian with unknown variance, Exponential).
For non-parametric families the calculation is much more challenging. The following theorem takes the first steps towards understanding this problem for the class of
distributions $\cH_{\kappaz}$ for $\kappaz \geq 7/2$.

\begin{theorem}\label{thm:kurtosis}
Let $\kappaz \geq 7/2$ and $\Delta > 0$ and $\nu \in \cH_{\kappaz}$ with mean $\mu$, variance $\sigma^2 > 0$ and kurtosis $\kappa$.
Then
\eq{
&\inf\set{
\KL(\nu, \nu') : \nu' \in \cH_{\kappa} \text{ and } \E_{X \sim \nu'}[X] > \mu + \Delta} \\
&\qquad \qquad \leq 
  \begin{cases}
    \min\set{\log\left(\frac{1}{1 - p}\right),\, \frac{C' \Delta^2}{\sigma^2}} & \text{if }\, C \kappa^{1/2}(\kappa+1)\frac{\Delta}{\sigma} < \kappaz \\[0.4cm]
    \log\left(\frac{1}{1 - p}\right) & \text{otherwise}\,,
  \end{cases}
}
where $C, C' > 0$ are universal constants and $p = \min\set{\Delta/\sigma, 1/\kappaz}$.
\end{theorem}

Notice that the result is strongest on the `interior' of $\cH_{\kappaz}$ (that is, when $\kappa \ll \kappaz$). 
In fact, this is necessary because $\cH_{\kappaz}$ includes the Bernoulli with kurtosis $\kappaz$ and 
in this case there is very little wiggle room available to perturb the mean of the measure without also increasing the kurtosis.
Since $\log(1 + x) \leq x$ for all $x$ we have
\eq{
\frac{1}{\log\left(\frac{1}{1-p}\right)} \geq \frac{1-p}{p} = \Omega\left(\kappaz + \frac{\sigma}{\Delta}\right)\,.
}
This means that provided $\kappa$ and $\Delta$ are sufficiently small relative to $\kappaz$, then the lower bound derived from the above 
theorem and \cref{eq:lower} matches the upper bound in the previous section up
to constant factors. The proof of Theorem \ref{thm:kurtosis} involves explicit alternative distributions $\nu'$ based on $\nu$ and is given in \cref{app:thm:kurtosis}.

\section{Summary}\label{sec:summary}

The assumption of finite kurtosis generalises the parametric Gaussian assumption to a comparable non-parametric setup with a similar basic structure.
Of course there are several open questions.

\paragraph{Optimal constants}
The leading constants in the main results (Theorem~\ref{thm:upper} and Theorem~\ref{thm:kurtosis}) are certainly quite loose.
Deriving the optimal form of the regret is an interesting challenge, with both lower and upper bounds appearing quite non-trivial. 
It may be necessary to resort to an implicit analysis showing that (\ref{eq:lower}) is (or is not) achievable when $\cH$ is the class of distributions with kurtosis bounded by some $\kappaz$.
Even then, constructing an efficient algorithm would remain a challenge.
Certainly what has been presented here is quite far from optimal.
At the very least the median-of-means estimator needs to be replaced, or the analysis improved. An excellent candidate is Catoni's estimator \citep{Cat12}, which is slightly more
complicated than the median-of-means, but also comes with smaller constants and could be plugged into the algorithm with very little effort.
For the lower bound, there appears to be almost no work on the explicit form of the lower bounds presented by \cite{BK96} in interesting non-parametric classes beyond
rewards with bounded or semi-bounded support \citep{HT10,HT15}.

\paragraph{Non-parametric Thompson sampling}
If an appropriate prior is used, then Thompson sampling has recently been shown to achieve the optimal rate when the distributions are Gaussian with unknown means and variances \cite{HT14}.
It is natural to ask if this algorithm can be generalised to the non-parametric setting discussed here. Note that this is possible in the case where the rewards
have bounded support \citep{KKM12}.

\paragraph{Absorbing other improvements}
There has recently been a range of improvements to the confidence level for the classical upper confidence bound algorithms that shave logarithmic terms from
the worst-case regret or improve the lower-order terms in the finite-time bounds \citep{AB09,Lat17a}. Many of these enhancements can be incorporated into the algorithm
presented here, which may lead to practical and theoretical improvements.

\paragraph{Replacing median-of-means with self-normalised inequalities}
While the median-of-means led to the simple analysis presented here, there is another approach that has the potential to lead to significantly smaller constants, which is
to use the theory of self-normalised processes \citep{PLS08}.

\paragraph{Comparison to Bernoulli}
Table \ref{table:kurtosis} shows that the kurtosis for a Bernoulli random variable with mean $\mu$ is $\kappa = O(1/(\mu(1-\mu)))$, which is obviously not bounded as $\mu$ tends towards
the boundaries. The optimal asymptotic regret for the Bernoulli case is
\eq{
\lim_{n\to\infty} \frac{\Re_n}{\log(n)} = \sum_{i:\Delta_i > 0} \frac{\Delta_i}{d(\mu_i, \mu^*)}\,.
}
The interesting differences occur near the boundary of the parameter space.
Suppose that $\mu_i \approx 0$ for some arm $i$ and $\mu^* > 0$ is close to zero.
An easy calculation shows that $d(\mu_i, \mu^*) \approx \log(1/(1 - \Delta_i)) \approx \Delta_i$. 
Therefore
\eq{
\liminf_{n\to\infty} \frac{\E[T_i(n)]}{\log(n)} \approx \frac{1}{\log(1/(1-\Delta_i))} \approx \frac{1}{\Delta_i}\,.
}
Here we see an algorithm is enjoying logarithmic regret on a class with infinite kurtosis! But this is a very special case and is 
not possible in general, as demonstrated by Theorem~\ref{thm:lower}.
The reason is that the structure of the hypothesis class allows strategies to (essentially) estimate the kurtosis with reasonable accuracy and anticipate outliers 
more/less depending on the data observed so far.

\bibliography{all}

\appendix

\section{Proof of Theorem \ref{thm:kurtosis}}\label{app:thm:kurtosis}
Assume without loss of generality that $\nu$ is centered and has variance $\sigma^2 = 1$, which can always be achieved by 
shifting and scaling (neither effects the kurtosis or the relative entropy). 
The result is proved by piecing together two ideas. The first idea is to perturb the distribution by adding a Bernoulli `outlier'. 
The second idea is to perturb the distribution more smoothly.
Let $X$ be a random variable sampled from $\nu$ and $B$ be a Bernoulli with parameter $p = \min\set{\Delta, 1/\kappaz}$.
Let $Z = X + Y$ where $Y = \Delta B / p$. Then $\E[Z] = \Delta$ and
\eq{
\Kurt[Z] 
&= 3 + \frac{\kappa - 3 + \Var[Y]^2 (\Kurt[Z] - 3)}{(1 + \Var[Y])^2} \\
&= 3 + \frac{\kappa - 3 + \left(\frac{(1-p)^2\Delta^2}{p}\right)^2 \frac{1-6p(1-p)}{p(1-p)}}{\left(1 + \frac{(1-p)^2\Delta^2}{p}\right)^2} \\
&\leq 3 + \frac{\kappaz - 3 + \left(\frac{(1-p)^2\Delta^2}{p}\right)^2 \frac{1-6p(1-p)}{p(1-p)}}{\left(1 + \frac{(1-p)^2\Delta^2}{p}\right)^2}
\leq \kappaz\,,
}
where the first inequality used Lemma~\ref{lem:kurtosis} and the final inequality follows from calculus and the assumption that $\kappaz \geq 7/2$.
Let $\nu' = \cL(Y)$ be the law of $Y$. Then
\eq{
\KL(\nu, \nu') \leq \log\left(\frac{1}{1-p}\right)\,.
}
Moving onto the second idea, where I use $C$ for a universal positive constant that changes from equation to equation.
Let $A = \set{x : |x| \leq \sqrt{a\kappa}}$ and $\bar A = \R - A$. Define alternative measure $\nu'(E) = \int_E (1 + g(x)) d\nu(x)$ where 
\eq{
g(x) = (\alpha + \beta x) \ind{x \in A}
}
for some constants $\alpha$ and $\beta$ chosen so that
\eq{
\int_{\R} g(x) d\nu(x) &= \alpha \int_A d\nu(x) + \beta \int_A x d\nu(x) = 0\,. \\
\int_{\R} g(x) x d\nu(x) &= \alpha \int_A x d\nu(x) + \beta \int_A x^2 d\nu(x) = \Delta\,.
}
Solving for $\alpha$ and $\beta$ shows that
\eq{
\beta = \frac{\Delta}{\int_{A} x^2 d\nu(x) - \frac{\left(\int_{A} x d\nu(x)\right)^2}{\nu(A)}} \quad \text{ and } \quad 
\alpha = -\frac{\Delta \int_A xd\nu(x)}{\nu(A) \int_{A} x^2 d\nu(x) - \left(\int_{A} x d\nu(x)\right)^2}\,. 
}
We still need to show that $\nu'$ is a probability measure, which will follow from the positivity of $1 - g(\cdot)$.
The first step is to control each of the terms appearing in the definitions of $\alpha$ and $\beta$.
By Cauchy-Schwarz and Chebyshev's inequalities,
\eq{
\nu(\bar A) = \nu(x^2 \geq a \kappa) \leq \frac{1}{\kappa a^2}
}
and
\eq{
\int_A x^2 d\nu(x)
= 1 - \int_{\bar A} x^2 d\nu(x) 
\geq 1 - \sqrt{\kappa \nu(\bar A)}
\geq 1 - \frac{1}{a}\,.
}
Similarly,
\eq{
\left|\int_A x d\nu(x)\right|
= \left|\int_{\bar A} x d\nu(x)\right|
\leq \sqrt{\sigma^2 \nu(\bar A)} 
\leq \frac{1}{a \sqrt{\kappa}}\,.
}
Therefore by choosing $a = 2$ we have\,,
\eq{
\left|\alpha\right| 
&= \Delta\left|\frac{\int_A x d\nu(x)}{\nu(A) \int_A x^2 d\nu(x) - \left(\int_A x d\nu(x)\right)^2}\right|
\leq \frac{\Delta/\sqrt{\kappa}}{a\left(\left(1 - \frac{1}{\kappa a^2}\right) \left(1 - \frac{1}{a}\right) - \frac{1}{a^2 \kappa}\right)} \leq \frac{4\Delta}{\sqrt{\kappa}} \\
\left|\beta\right|
&= \frac{\Delta}{\int_A x^2 d\nu(x) - \frac{\left(\int_A x d\nu(x)\right)^2}{\nu(A)}}
\leq \frac{\Delta}{1 - \frac{1}{a} - \frac{1}{\kappa a^2\left(1 - \frac{1}{a^2\kappa}\right)}}
\leq 6\Delta\,.
}
Now $g(x)$ is an increasing linear function supported on $A$, so 
\eq{
\max_{x \in \R} |g(x)| 
&= \max\set{|g(\sqrt{a\kappa})|, |g(-\sqrt{a\kappa})|} 
\leq |\alpha| + \sqrt{a\kappa} |\beta|
\leq \frac{4\Delta}{\sqrt{\kappa}} + 6\Delta \sqrt{2\kappa}
\leq \frac{1}{2}\,,
}
where the last inequality by assuming that 
\eq{
\Delta \leq \frac{\sqrt{\kappa}}{4(2 + 3\sqrt{2}\kappa)} = O(\kappa^{-1/2})\,,
}
which is reasonable without loss of generality, since if $\Delta$ is larger than this quantity, then we would prefer the bound that depends on $\kappaz$ derived
in the first part of the proof. The relative entropy between $\nu$ and $\nu'$ is bounded by 
\eq{
&\KL(\nu, \nu') 
\leq\chi^2(\nu, \nu') 
= \int_\R \left(\frac{d\nu(x)}{d\nu'(x)} - 1\right)^2 d\nu'(x) 
= \int_A \frac{g(x)^2}{1 + g(x)} d\nu(x) \\
&\quad\leq 2 \int_A g(x)^2 d\nu(x) 
\leq 4 \int_A \alpha^2 d\nu(x) + 4\int_A \beta^2 x^2 d\nu(x) 
\leq 4 \alpha^2 + 4\beta^2 \\
&\leq \frac{4 \cdot 16 \Delta^2}{\kappa} + 4 \cdot 36 \Delta^2 \leq C \Delta^2\,.
}
In order to bound the kurtosis we need to evaluate the moments:
\eq{
\int_\R x^2 d\nu'
&= \int_\R x^2 d\nu + \int_A g(x) x^2 d\nu
= 1 + \alpha \int_A x^2 d\nu(x) + \beta \int_A x^3 d\nu(x) \\
&\leq 1 + C \Delta \sqrt{\kappa}\,. \\
\int_\R x^2 d\nu'
&= \int_\R x^2 d\nu + \int_A g(x) x^2 d\nu
\geq 1 - C \Delta \sqrt{\kappa}\,. \\
\int_\R x^4 d\nu'
&= \int_\R x^4 d\nu + \int_A g(x) x^4 d\nu
= \kappa + \alpha \int_A x^4 d\nu(x) + \beta \int_A x^5 d\nu(x) \\
&\leq \kappa\left(1 + C\Delta \sqrt{\kappa}\right)\,. \\ 
\left|\int_\R x^3 d\nu'(x)\right|
&\leq \sqrt{\int_\R x^2 d\nu'(x) \int_\R x^4 d\nu'(x)} 
\leq \sqrt{C\kappa}\,. 
}
Therefore if $\kappa'$ is the kurtosis of $\nu'$, then
\eq{
\kappa' = \frac{\int_\R (x - \Delta)^4 d\nu'(x)}{\left(\int_\R x^2 d\nu'(x) - \Delta^2 \right)^2}
= \frac{\int_\R x^4 d\nu'(x) - 3\Delta^4 + 6\Delta^2 \int_\R x^2 d\nu'(x) - 4\Delta \int_\R x^3 d\nu'(x)}{\left(1 - \Delta^2 + \alpha \int_A x^2 d\nu(x) + \beta \int_A x^3 d\nu(x)\right)^2}
}
As a brief aside, if $\nu$ is symmetric, then the odd moments vanish and $\int_A x^i d\nu(x) = 0$ for odd $i$. Therefore $\alpha = 0$ 
\eq{
\kappa' = \frac{\kappa - 3\Delta^4 + 6\Delta^2}{(1 - \Delta^2)^2}
\leq \frac{\kappa + 6\Delta^2}{1 - 2\Delta^2}
= \kappa + \frac{6\Delta^2}{1 - 2\Delta^2} + \frac{2\kappa \Delta^2}{1 - 2\Delta^2}
\leq \kappa + C\kappa \Delta^2\,.
}
On the other hand, if $\nu$ is not symmetric, then the odd moments must be controlled.
\eq{
\kappa' 
&= \frac{\int_\R x^4 d\nu'(x) - 3\Delta^4 + 6\Delta^2 \int_\R x^2 d\nu'(x) - 4\Delta \int_\R x^3 d\nu'(x)}{\left(\int_\R x^2 d\nu'(x) - \Delta^2\right)^2}\\
&\leq \frac{\kappa\left(1 + C\Delta \kappa^{1/2}\right) + 6\Delta^2(1 + C \Delta \kappa^{1/2}) + C\Delta \kappa^{1/2}}{\left(1 - C\Delta \kappa^{1/2} - \Delta^2\right)^2}\\
&\leq \frac{\kappa + C\Delta \kappa^{1/2}(\kappa+1)}{1 - C\Delta \kappa^{1/2}}
\leq \kappa + C\Delta \kappa^{1/2}(\kappa+1)\,. 
}
By patching the two results we obtain that for all $\Delta > 0$ and $\nu \in \cH_{\kappaz}$ with mean $\mu$, variance $\sigma^2 > 0$ and kurtosis $\kappa$,
\eq{
&\inf\set{
\KL(\nu, \nu') : \nu' \in \cH_{\kappa} \text{ and } \E_{X \sim \nu'}[X] > \mu + \Delta} \\
&\qquad \qquad \leq 
  \begin{cases}
    \min\set{\log\left(\frac{1}{1 - p}\right),\, \frac{C' \Delta^2}{\sigma^2}} & \text{if }\, C \kappa^{1/2}(\kappa+1)\frac{\Delta}{\sigma} < \kappaz \\[0.4cm]
    \log\left(\frac{1}{1 - p}\right) & \text{otherwise}\,,
  \end{cases}
}
where $C, C' > 0$ are universal constants and $p = \min\set{\Delta/\sigma, 1/\kappaz}$.

\end{document}